\documentclass{article}
\usepackage{iclr2027_conference,times}
\usepackage{amsmath,amssymb,mathtools,amsthm}
\usepackage{booktabs,tabularx,array,multirow}
\usepackage{graphicx,xcolor,placeins}
\usepackage[expansion=false]{microtype}
\usepackage{xspace,url,hyperref}
\usepackage[nameinlink,capitalize,noabbrev]{cleveref}
\hypersetup{pdftitle={CADOC: Cache-Aware Dynamic Object Context for Long-Horizon Agents},pdfauthor={Anonymous authors},colorlinks=true,linkcolor=blue!50!black,citecolor=blue!50!black,urlcolor=blue!50!black}
\graphicspath{{figures/}}
\newtheorem{proposition}{Proposition}

\newcommand{\cost}{\mathcal{C}}
\newcolumntype{Y}{>{\raggedright\arraybackslash}X}
\newcolumntype{L}[1]{>{\raggedright\arraybackslash}p{#1}}
\title{CADOC: Cache-Aware Dynamic Object Context for Long-Horizon Agents}

\author{
Junjie Yao$^{1}$, \quad
Zhangchen Zhou$^{1}$, \quad
Zhi-Qin John Xu$^{1,2,}$\thanks{Corresponding author: \texttt{xuzhiqin@sjtu.edu.cn.}}
\\[2mm]
$^{1}$School of Mathematical Sciences, Shanghai Jiao Tong University, \\Shanghai, China.\\
$^{2}$Institute of Natural Sciences, Shanghai Jiao Tong University, \\Shanghai, China.
}

\iclrfinalcopy

\begin{document}
\maketitle
\pagestyle{plain}
\suppressfloats[t]

\begin{abstract}
For a long-horizon agent, context is the bottleneck: the history is resent with every request, the window caps task length, and reasoning degrades as the history grows. Replacing structured objects with compact retrieval Cards shortens the prompt and keeps the exact originals retrievable, but editing the history can break prefix-cache reuse, and prior recoverable methods time their edits by forecasts of future reuse or by preset intervals. We propose CADOC (Cache-Aware Dynamic Object Context), an online algorithm that replaces structured objects with compact Cards while preserving exact, on-demand retrieval of their original contents. CADOC schedules replacements in batches by balancing accumulated waiting cost against shared cache-reconstruction cost. Its scheduling rule follows from an economic order quantity trade-off, recovers the optimal integer batch under stationary assumptions. Across evaluation, CADOC consistently achieves the lowest aggregate input cost among the compared configurations, which reduces input cost by approximately 40\% on average while maintaining task performance close to full context. CADOC thus provides a cost-derived approach to compressible context management, demonstrating that efficient compression depends not only on shortening prompts but also on scheduling edits to preserve cache reuse.
\end{abstract}

\section{Introduction}
\label{sec:intro}

Context is the bottleneck for long-horizon agents. Every request resends the history, so retained tokens incur costs again, even on a cache hit. The context window limits task length, and models struggle to use relevant information in long histories~\citep{kang2025acon,sun2025contextfolding}. Managing context therefore affects both cost and performance.

Compression is one remedy, and much of what accumulates invites it: code, file contents, logs, and tool outputs matter when produced but seldom need to stay in view. Lossy methods, whether token pruning, learned agent compression, or subtrajectory folding~\citep{jiang2023llmlingua,kang2025acon,sun2025contextfolding}, discard details a later step may need; a structured object can instead be moved to external storage~\citep{packer2023memgpt} and replaced by a compact retrieval Card that keeps the exact original retrievable (\cref{fig:overview}b).

\begin{figure}[hbpt]
\centering
\includegraphics[width=\linewidth]{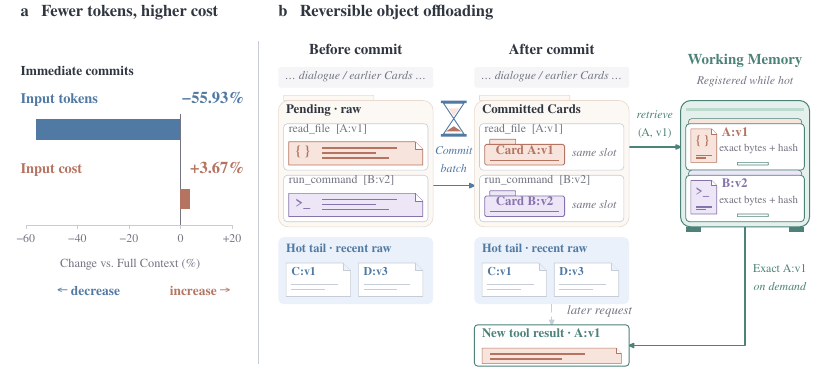}
\caption{(a) Immediate cuts cumulative input tokens by 55.93\% but raises input cost by 3.67\% versus Full Context. (b) Pending objects become Cards in place; the hot tail stays visible, and Working Memory returns exact contents.}
\label{fig:overview}
\end{figure}

A recoverable method must still decide when to commit; the two closest systems account for the prefix cache. Self-GC commits the edits a planner LLM proposes once their expected saving over future requests outweighs the cache break, so it must estimate future reuse~\citep{hao2026selfgc}. TokenPilot runs an LLM estimator every fixed, empirically chosen number of turns and evicts segments it judges expired~\citep{xu2026tokenpilot}. Neither sets the timing from what is already observable. Editing an earlier object can invalidate the cached suffix after it, including the raw \emph{hot tail} of recent blocks, so separate commits repeat the same reconstruction while a batch shares it; and waiting keeps payloads in every intervening request. On an example from Hermes Compression Evaluation (HCE) histories~\citep{nous2026hce}, we find that immediate replacement removes 55.93\% of cumulative input tokens yet raises input cost by 3.67\% (\cref{fig:overview}a).

CADOC, \emph{Cache-Aware Dynamic Object Context}, is a fully online, LLM-free scheduler that determines when to commit object replacements in batches. It requires \emph{no future information or workload forecasts, no additional LLM calls, and no per-setting batch-size tuning}: every scheduling decision uses only current shortening and cache estimates. Rather than fixing a batching interval in advance, CADOC dynamically balances the cache-reconstruction cost that a batch can share against the cost accumulated by delaying compression, and commits when further waiting no longer pays.

Our experiments demonstrate two complementary advantages of CADOC. First, \emph{CADOC reduces agent input costs by approximately 40\% while maintaining comparable task performance}. We evaluate the complete system through continuing task chains on several benchmarks, retaining conversation history across tasks to capture the costs of long-horizon execution. Relative to a full-context baseline, CADOC reduces input costs substantially. Second, \emph{CADOC consistently achieves the lowest cost among the compared commit policies across diverse workloads and hyperparameter settings}. We evaluate commit scheduling through fixed-history replays, comparing CADOC with immediate replacement, fixed-size batching, and token-threshold policies. CADOC attains the lowest cost at all settings. Whereas the best fixed batch size changes across settings, CADOC adapts its commit timing automatically, obtaining these results without future information, additional LLM calls for scheduling, or per-setting batch-size tuning.

\FloatBarrier
\section{Objects, Cards, and the Hot Tail}
\label{sec:substrate}

CADOC shortens the visible context by reversible object externalization, not lossy compression: Working Memory keeps the original bytes, and a Card stands in for the object in the prompt. Registration and replacement preserve the object’s original position in the interaction history; retrieval returns its immutable version as a new tool-result occurrence.

\paragraph{Objects and immutable versions.}
Runtime metadata and parsing rules identify structured spans such as file contents, code, tool results, logs, tables, and generated artifacts. Reliably identified objects whose replacement reduces tokens are registered before replacement; ordinary dialogue stays visible. Each stored object has an immutable version, so a historical Card recovers the contents observed at that point even if the file is modified later.

\paragraph{Cards and exact retrieval.}
A Card records the object's identity, version, source, and structural metadata, with no LLM-generated synopsis. The Card below replaced a terminal search result in a recorded task; ellipses mark omitted metadata:
\begin{quote}\begin{minipage}{\linewidth}\small\ttfamily\raggedright
<OBJECT\_CARD>\\
\{"contains": \{..., "top\_level\_keys":\\
\quad ["output", "exit\_code", "error"]\}, ...,\\
\quad "object\_ref": "object://obj\_0722dfa749e04b58a0094af4@v1",\\
\quad "origin": \{..., "tool": "terminal"\}, ...,\\
\quad "type": "structured\_data", "version": 1\}\\
</OBJECT\_CARD>
\end{minipage}\end{quote}
The model calls \texttt{retrieve\_object}, an ordinary function tool taking \texttt{object\_ref} and \texttt{reason}; \cref{app:cards} shows the recorded call, recovery check, and subsequent action.

\paragraph{Hot tail and pending sequence.}
An \emph{inference block} is a model response together with its tool results. CADOC keeps a recent raw suffix bounded by a maximum block count $h$ and a token budget $T_{\mathrm{hot}}$. Once an object leaves this protected suffix, it joins the pending sequence and stays raw until a \emph{batch commit} replaces its prepared span with a Card.

\section{Cache-Aware Batch Scheduling}
\label{sec:scheduling}

Batching lets several object replacements share the cost of reconstructing the cached hot tail, but waiting to form a batch keeps removable payloads in every intervening input. We derive the batch size that balances these costs in a stationary stream, then restate its optimality condition in quantities available during execution; the result is the online rule CADOC uses. \Cref{fig:commit} summarizes the affected interval and the decision.

\begin{figure}[t]
\centering
\includegraphics[width=\linewidth]{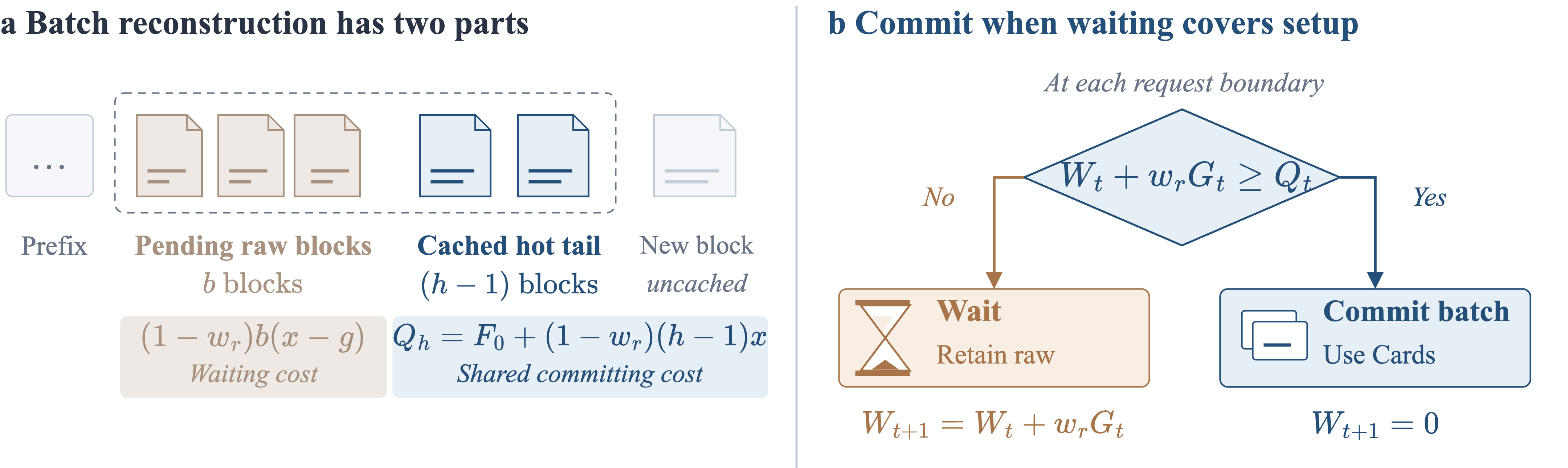}
\caption{The cached interval and the batch-commit rule. Left: the $b$ pending blocks shorten to $b(x-g)$ tokens, while the $h-1$ cached hot blocks stay raw. Right: waiting-loss updates and the commit decision.}
\label{fig:commit}
\end{figure}

\subsection{Replacement Changes Both Prompt Length and Cache Reuse}
\label{sec:cachecost}

Normalize the cost of an uncached input token to one, and let the cache-read weight $w_r\in[0,1]$ be the relative cost of a cached token. Total input cost is
\begin{equation}
\cost=\sum_t(N_t+w_rR_t),
\label{eq:inputcost}
\end{equation}
where $N_t$ and $R_t$ count the uncached and cached input tokens of request $t$. The main experiments use $w_r=0.1$.

Consider replacements inside an otherwise reusable prefix. Let $I$ be the number of cached tokens from the first edit to the cache frontier, and $G$ the number of tokens removed. Keeping the interval costs $w_rI$ on the next request; replacing it requires fresh processing of $I-G$ tokens. The first-request cost difference is therefore
\begin{align}
\Delta\cost_{\mathrm{first}}
&=(I-G)-w_rI+F \nonumber\\
&=\underbrace{(1-w_r)(I-G)+F}_{\text{reconstruction premium}}
-\underbrace{w_rG}_{\text{shortening benefit}}.
\label{eq:firstcost}
\end{align}
Here $F$ is an optional auxiliary commit cost in the same units. The input-token accounting of \cref{eq:inputcost} has no auxiliary costs, so $F=0$ there; Card tokens are already part of $G$ and are not charged again through $F$.

\subsection{Stationary Model and the EOQ Correspondence}
\label{sec:eoq}

Consider one new block per request cycle, of fixed length $x$ and containing one replaceable object whose replacement shortens the block by $g\in(0,x)$ tokens. The token budget is inactive here, so the hot tail holds exactly $h\geq1$ blocks, and eligible objects have completed their initial raw exposure. Each commit replaces all $b$ pending blocks. We assume stable prefix caching, no retrieval, and an indefinitely continuing stream; in this block model the editable span starts at the beginning of each block.

Decisions are made after a block completes and before the next request. At that point the newest block has not yet appeared in any input, while the preceding $h-1$ hot blocks are cached, so a commit affects a cached interval of length
\begin{equation}
I_b=\underbrace{bx}_{\text{pending blocks}}+\underbrace{(h-1)x}_{\text{cached hot tail}}.
\label{eq:Ib}
\end{equation}
The newest block needs fresh processing under either decision, so its cost cancels from the comparison.

To compare batch sizes on a common basis, we use an idealized reference that replaces every object as soon as it is eligible and gets a cache hit on the new Card representation without reconstruction. The reference is an accounting device: relative to it, a real schedule pays for delayed shortening and for reconstruction, and all other input costs are common.

An object retained for one extra cached request costs $r=w_rg$. The $b$ objects of a batch wait $b-1,b-2,\ldots,0$ requests after becoming eligible, so their total waiting loss is
\begin{equation}
W_b=r\sum_{k=0}^{b-1}k=\frac{rb(b-1)}{2}.
\label{eq:waiting-loss}
\end{equation}
At the commit, schedule and reference both use Cards, so the only extra cost is the reconstruction premium of \cref{eq:firstcost}, namely $(1-w_r)(I_b-bg)+F(b)$. Write $F(b)=F_0+bf$, with shared overhead $F_0\geq0$ and per-object overhead $f\geq0$. The average excess cost per object is
\begin{align}
\ell_h(b)
&=\frac{W_b+(1-w_r)(I_b-bg)+F_0+bf}{b} \nonumber\\
&=\underbrace{f+(1-w_r)(x-g)}_{K}
+\frac{\overbrace{F_0+(1-w_r)(h-1)x}^{Q_h}}{b}
+\frac{r(b-1)}{2}.
\label{eq:eoq}
\end{align}
Thus,
\begin{equation}
\ell_h(b)=K+\frac{Q_h}{b}+\frac{r(b-1)}{2},\qquad r=w_rg.
\label{eq:components}
\end{equation}
The pending blocks contribute a constant $K$ per object. Batching lowers the per-object share of the shared cost, $Q_h/b$, but raises the per-object waiting loss $r(b-1)/2$: the EOQ trade-off between setup and holding costs~\citep{harris1913parts,erlenkotter1990harris}.

\begin{proposition}[Stationary optimal batch]
\label{prop:eoq}
For $Q_h\geq0$ and $r>0$, the continuous minimizer over $b\geq1$ is
\begin{equation}
b_{\mathrm{cont}}^*=\max\left\{1,\sqrt{\frac{2Q_h}{r}}\right\}.
\label{eq:sqrt}
\end{equation}
A positive integer $b$ minimizes $\ell_h$ if and only if
\begin{equation}
\frac{rb(b-1)}{2}\leq Q_h\leq\frac{rb(b+1)}{2}.
\label{eq:integer}
\end{equation}
\end{proposition}

The square-root expression follows from $\ell_h'(b)=-Q_h/b^2+r/2$. For integer batches the relevant comparison is
\begin{equation}
\ell_h(b+1)-\ell_h(b)=\frac r2-\frac{Q_h}{b(b+1)}.
\label{eq:batch-difference}
\end{equation}
These differences are nondecreasing in $b$, so a batch is optimal exactly when enlarging it no longer lowers the cost and the previous enlargement did not raise it; these two conditions are \cref{eq:integer}. \Cref{app:proofs} gives the full accounting, proofs, and boundary cases. Larger shared reconstruction costs favor larger batches; larger per-request savings favor earlier commits.

\subsection{From the Stationary Optimum to an Online Rule}
\label{sec:online}

In the stationary model, a pending batch of size $b$ offers shortening $G_b=bg$, and its waiting loss satisfies
\begin{equation}
W_b+w_rG_b=\frac{rb(b-1)}2+rb=\frac{rb(b+1)}2=W_{b+1}.
\label{eq:stationary-crossing}
\end{equation}
Hence \cref{eq:integer} reads $W_b\leq Q_h\leq W_b+w_rG_b$. Start from an empty pending sequence and take the first $b$ at which $W_b+w_rG_b\geq Q_h$. If $b>1$, the previous decision did not cross, so $W_b<Q_h$; if $b=1$, then $W_1=0\leq Q_h$. The first crossing therefore selects an optimal integer batch, the smaller one when adjacent sizes tie.

For a changing history, CADOC evaluates the same crossing with current quantities. At boundary $t$, let $\mathcal P_t$ be the set of pending object occurrences. Their potential shortening is
\begin{equation}
G_t=L_t^{\mathrm{keep}}-L_t^{\mathrm{commit}},
\label{eq:pending-gain}
\end{equation}
where the two lengths describe the next prompt with all pending contents retained or replaced; under additive token accounting this is the sum of raw-to-Card reductions. Let $W_t$ be the waiting loss accumulated since the current pending contents became eligible; newly eligible contents enter with zero loss, and carrying the pending contents through one more cached request adds
\begin{equation}
\Delta W_t=w_rG_t.
\label{eq:waitincrement}
\end{equation}

Let $H_t^{\mathrm{shared}}$ count the already cached tokens after the newest pending block and before the reusable cache frontier. These tokens survive the replacements but must be reconstructed with them, so the estimated shared reconstruction cost is
\begin{equation}
Q_t=F_{0,t}+(1-w_r)H_t^{\mathrm{shared}}.
\label{eq:Qonline}
\end{equation}
This reduces to $Q_h$ in the stationary model. The recorded replay uses $F_{0,t}=0$; \cref{app:online-scope} details how this suffix is separated from the pending region.

For nonempty $\mathcal P_t$, the \emph{economic crossing} is
\begin{equation}
\boxed{\text{commit all pending replacements if }W_t+w_rG_t\geq Q_t.}
\label{eq:trigger}
\end{equation}
Otherwise the controller waits and adds $w_rG_t$ to the ledger after each request that carries the pending contents raw; committed entries are cleared once a successful model response confirms the commit. Hard capacity limits can force an earlier commit. Since the earliest pending edit fixes where reconstruction begins, committing the later pending replacements in the same request shortens the interval at no additional cache break, which is why CADOC commits the whole pending sequence at once.

The evaluated runtime estimates $G_t$ from rough keep and commit prompt lengths, $Q_t$ from prompt and cache estimates, and $W_t$ from a persistent per-delta waiting ledger; the recorded decisions in \cref{app:runtime} use the same symbols for these estimates.

Under the stationary assumptions, the crossing recovers the optimum exactly; in changing histories it adapts commit timing to the observed shortening and cache estimates. For the associated waiting--setup objective with fixed eligibility, \cref{app:reduced-batching-bound} also proves a finite-horizon competitive bound.

\section{Fixed-History Evaluation of Commit Scheduling}
\label{sec:hce}

\subsection{Protocol and Compared Policies}

HCE replay~\citep{nous2026hce} compares commit policies on identical source histories. Model outputs, object versions, and Card contents are frozen, so each configuration yields a prompt stream whose lengths and reusable prefixes determine its cost. We evaluated at $h\in\{2,4,8,16\}$. All costs follow \cref{eq:inputcost} with $w_r=0.1$; \cref{app:price} examines other cache-read weights.

We compare Full Context, Immediate, fixed batches $b\in\{2,4,8,16\}$, five token thresholds from 4,096 to 65,536, and CADOC. Full Context retains everything, Immediate commits at eligibility, and the remaining baselines wait for a block count, or a pending raw-token threshold. CADOC uses the object lifecycle of \cref{sec:substrate}, rough prompt and cache estimates, and a persistent waiting ledger. \cref{app:protocol} details the replay paths and metric definitions.

\subsection{Aggregate Cost and Sensitivity to Batch Size}
\label{sec:hce-rank}
\label{sec:transfer}

CADOC achieves the lowest aggregate input cost among all evaluated configurations at every tested hot-tail setting, without per-setting batch-size tuning (\cref{fig:hcescheduling}). Relative to Full Context, it saves 30.65\%, 20.61\%, 15.14\%, and 11.56\% at $h=2,4,8,16$, respectively.

\begin{figure}[htbp]
\centering
\includegraphics[width=\linewidth]{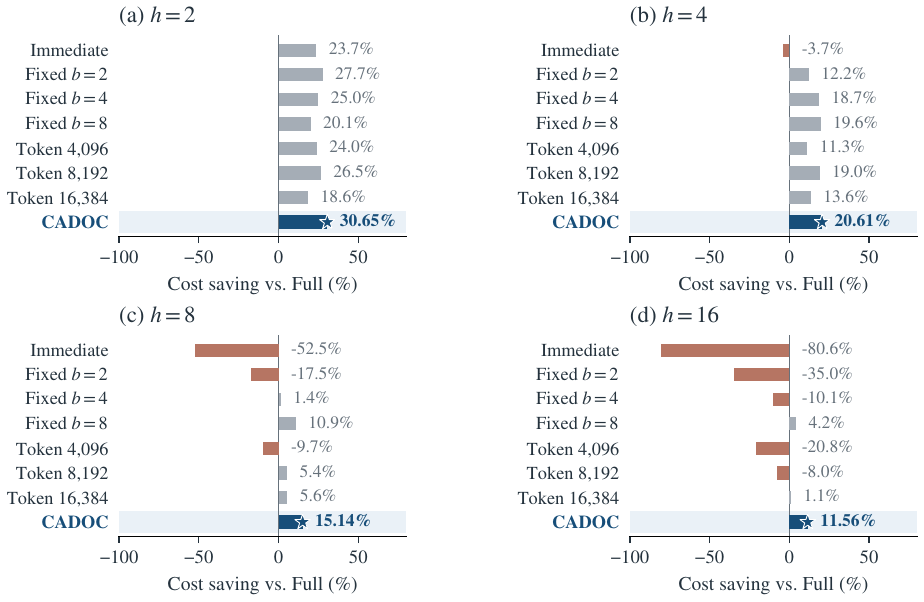}
\caption{Aggregate input-cost savings relative to Full Context over HCE. Negative values indicate costs above Full Context.}
\label{fig:hcescheduling}
\end{figure}

By contrast, a fixed batch size that works well at one hot-tail setting can lose its benefit or even increase cost at another (\cref{fig:transfer}a). The best tested fixed batch shifts from $b=2$ at $h=2$ to $b=8$ at $h=4,8,16$. Keeping $b=2$ increases cost relative to Full Context by approximately 17.5\% at $h=8$ and 35.01\% at $h=16$, while $b=4$ increases cost by approximately 10.1\% at $h=16$. CADOC avoids these reversals without manual batch-size adjustment and remains cheaper than the best tested fixed-batch configuration selected separately for each setting.

CADOC's recorded batch sizes make this adaptation explicit (\cref{fig:transfer}b). As $h$ increases, its mean batch size rises from 1.50 to 3.25, 4.75, and 5.00 blocks, while its total commit count falls from 20 to 3. This trend agrees with the stationary prediction in \cref{eq:sqrt}: a larger hot tail increases the shared cache-reconstruction cost, favoring larger batches over which to amortize that cost. CADOC automatically makes this adjustment through the same online rule, rather than requiring a manually chosen batch size for each setting. Aggregate continuation-quality results across the four hot-tail settings appear in \cref{app:quality-grid}. 

\begin{figure}[ht]
\centering
\includegraphics[width=\linewidth]{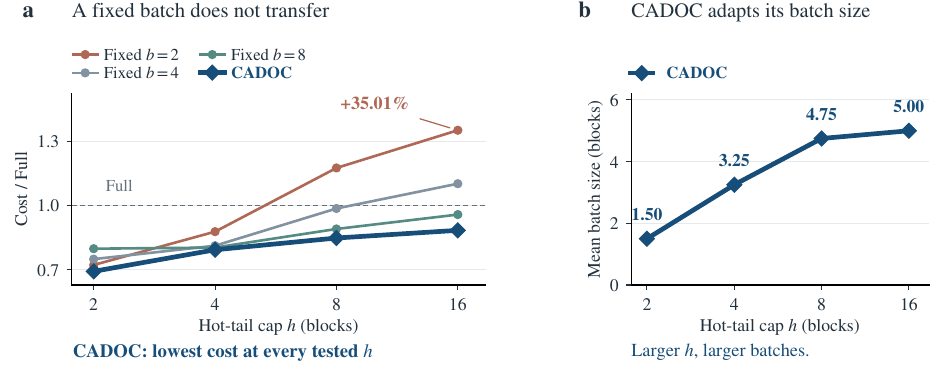}
\caption{Commit costs and batch sizes across hot-tail settings. (a) Aggregate input costs of fixed-batch configurations and CADOC, normalized by Full Context. (b) CADOC's mean number of blocks per committed batch at $h=2,4,8,16$.}
\label{fig:transfer}
\end{figure}
\vspace{-10pt}


\section{System-Level Evaluation on Benchmark Chains}
\label{sec:benchmarks}

To evaluate the complete system, we run CADOC and a full-context control arm, FULL, through the same fixed task lists with history retained across successive tasks; we call each such run a benchmark chain. Each arm develops its own model and tool trajectory, so total input cost reflects both the changed context representation and any change in model requests or retrieval. Costs cover every evaluated input request, including retrieved payloads and their later replacement by retrieval receipts, under \cref{eq:inputcost} with $w_r=0.1$. \Cref{app:protocol} gives the configuration and execution scope.

\begin{table}[ht]
\caption{Complete benchmark-chain results. Input denotes cumulative input tokens, Cost uses $N+0.1R$, and both reductions are relative to FULL.}
\label{tab:benchmarks}
\centering\small
\setlength{\tabcolsep}{4pt}
\begin{tabular}{lrrrr}
\toprule
Benchmark & FULL solved & CADOC solved & Input $\downarrow$ & Cost $\downarrow$\\
\midrule
Terminal-Bench 2.1~\citep{merrill2026terminalbench} & 38/89 & 36/89  & 49.40\% & 47.13\%\\
LongMemEval-V2~\citep{wu2026longmemevalv2} & 10/18 & 10/18  & 80.60\% & 42.40\%\\
SWE-bench~\citep{jimenez2024swebench} & 31/50 & 29/50  & 44.72\% & 34.12\%\\
\bottomrule
\end{tabular}
\end{table}
\subsection{Task Performance and Long-Horizon Cost Savings}
\label{sec:chain-curves}

CADOC substantially reduces input cost while maintaining task performance close to FULL (\cref{tab:benchmarks}). Task-completion rates decrease by at most four percentage points, whereas input cost falls by 34.12--47.13\% and cumulative input tokens by 44.72--80.60\%. Thus, CADOC achieves substantial savings with only small observed differences in task performance. 

The cumulative-cost curves in \cref{fig:chaincost} further show that CADOC's input cost grows more slowly overall than FULL's. As execution continues and conversation histories lengthen, the gap between the two curves becomes increasingly pronounced. Replacing historical objects with compact Cards avoids repeatedly carrying their full contents into subsequent requests, allowing savings from earlier replacements to accumulate. This growing separation highlights CADOC's advantage in long-horizon execution, where historical content would otherwise be processed repeatedly over many requests.

These savings remain substantial even when the agent actively retrieves externalized objects. On LongMemEval-V2, CADOC preserves every task-level outcome while recording 40 retrieval events over 146 model requests and reducing input cost by 42.40\%. The reported costs include requests carrying retrieved contents, so the savings account for the input overhead of restoring information when needed. The recorded case in \cref{sec:substrate,app:cards} further verifies byte-exact recovery and successful downstream use, showing that replacing objects with Cards preserves access to task-relevant content rather than discarding it.
\begin{figure}[htpb]
\centering
\includegraphics[width=\linewidth]{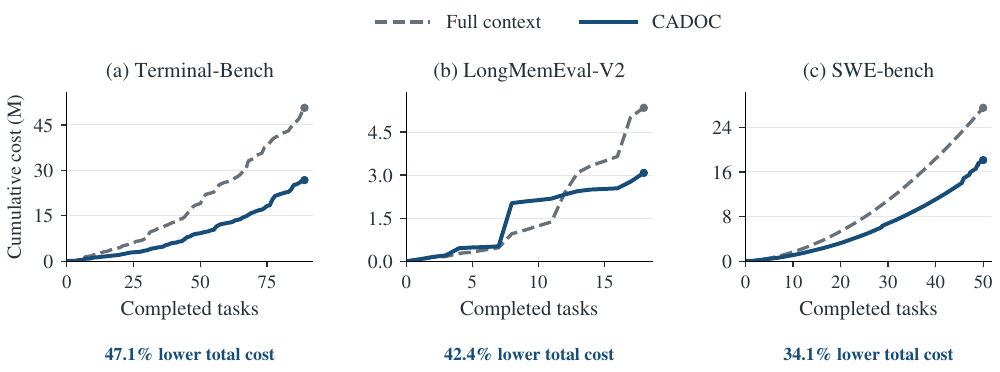}
\caption{Cumulative input cost during the three benchmark chains. The horizontal axis counts completed tasks. Gray dashed curves denote FULL and blue curves denote CADOC.}
\label{fig:chaincost}
\end{figure}

\subsection{Replay on Benchmark Histories}
\label{sec:benchmark-replay}

To compare compression policies on identical agent trajectories, we use the retained conversation histories exported from the FULL runs on Terminal-Bench, SWE-bench, and LongMemEval-V2. We then reconstruct the successive inference boundaries and replay each history in chronological order under different compression configurations, rather than compressing only the final context snapshot. During replay, each policy determines when to replace eligible objects; the resulting prompt lengths and reusable cache prefixes determine its input cost at each request. Context and cache state persist across task boundaries. This protocol compares the costs of different compression policies while holding the recorded agent behavior fixed.

CADOC achieves the lowest aggregate input cost among all evaluated configurations on all three benchmark histories (\cref{fig:benchmark-replay}). This consistent advantage contrasts with the workload-dependent effectiveness of other strategies. We further evaluate sensitivity to the cache-read weight $w_r$ (\cref{app:price}). Across all tested weights on all three benchmark histories, CADOC attains the lowest input cost among the compared policies. These results show that its cost advantage persists across both different workloads and different relative cache-read costs, rather than depending on a single pricing configuration. \Cref{app:replay-protocol} provides the replay protocol.

\begin{figure}[htbp]
\centering
\includegraphics[width=\linewidth]{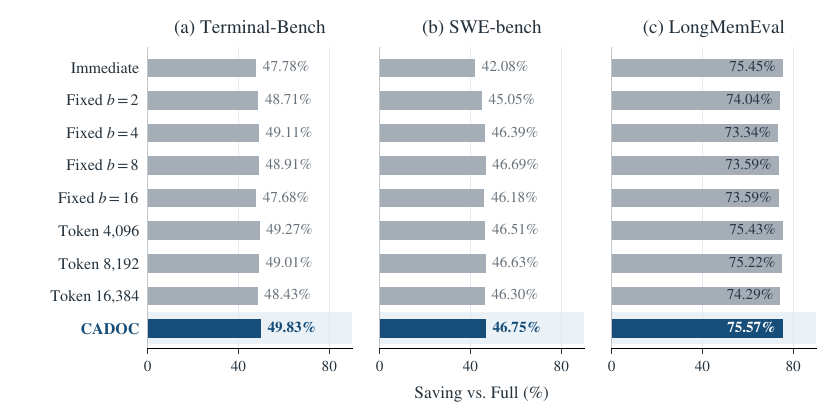}
\caption{Input-cost savings from offline replay of the conversation histories retained by FULL on three benchmarks. }
\label{fig:benchmark-replay}
\end{figure}

\section{Related Work}
\label{sec:related}

Agent context management spans prompt compression and observation masking~\citep{jiang2023llmlingua,lindenbauer2025complexity}, history consolidation, and external memory. ReSum and MEM1 maintain summaries or compact memory states~\citep{wu2025resum,zhou2025mem1}, ACON refines compression guidelines through failure analysis~\citep{kang2025acon}, and Context-Folding and AgentFold condense agent trajectories~\citep{sun2025contextfolding,ye2025agentfold}. ContextBudget explicitly learns when and how much to compress from the remaining context budget and incoming observation size~\citep{wu2026contextbudget}. Complementary systems such as MemGPT, Pichay, and LCM move information outside the active prompt while preserving access through memory tiers, demand paging, or summary hierarchies~\citep{packer2023memgpt,mason2026pichay,ehrlich2026lcm}. Beyond reducing prompt length, cache-aware approaches account for the cost implications of prefix reuse~\citep{lumer2026dontbreakcache}. Self-GC uses a side-channel semantic planner for recoverable context edits and weighs expected future savings against cache disruption and planning overhead~\citep{hao2026selfgc}. TokenPilot combines deterministic compaction and artifact recovery with model-based lifecycle estimation at an empirically selected check interval~\citep{xu2026tokenpilot}. CADOC instead operates on deterministic structured-object replacements, keeping surrounding dialogue visible. Its central distinction is cost-derived commit scheduling: replacement eligibility is separated from commit timing, and batches are triggered by accumulated waiting cost, current token shortening, and shared cache-reconstruction cost. This enables online adaptation without future-reuse forecasts and LLM calls.

\section{Conclusion}

We proposed CADOC, an online algorithm for recoverable, cache-aware context management that addresses the cost of disrupting prefix caches during compression. CADOC replaces structured objects with retrievable Cards and balances accumulated waiting loss against shared reconstruction cost to determine commit timing, without future information, LLM calls for Card construction or scheduling, or per-setting batch-size tuning. Its rule recovers the optimal integer batch under stationary assumptions. Experiments show the lowest aggregate input cost among evaluated policies across benchmark-history replays. On three continuing benchmark chains, CADOC reduces input cost by 34.12--47.13\% while maintaining task performance close to FULL. We thus provide a cost-derived approach to long-horizon context management, demonstrating that effective compression depends on both what is replaced and when replacements are committed.
\label{sec:main-end}

\clearpage
\section*{AI Use Statement}
Generative AI tools (ChatGPT with gpt-5.6-sol) assisted with drafting the manuscript and improving the writing.

\section*{Reproducibility Statement}
The supplementary material contains complete derivations, experimental configurations, aggregate evaluation results, a recorded retrieval case, and runtime diagnostics. Configuration details appear in \cref{app:protocol}, and proofs appear in \cref{app:proofs}.

\bibliographystyle{iclr2027_conference}
\bibliography{references}
\clearpage
\appendix
\section{Cost Accounting and Scheduling Proofs}
\label{app:proofs}

We derive the waiting--reconstruction decomposition, prove the stationary optimum and its first-crossing rule, and then state a finite-horizon bound for the online waiting--setup objective.

\subsection{A Common Reference for Comparing Schedules}
\label{app:cost-accounting}

Let $L_t=N_t+R_t$ be prompt length. Input cost satisfies
\begin{equation}
N_t+w_rR_t=w_rL_t+(1-w_r)N_t.
\label{eq:cost-decomposition}
\end{equation}
The first term charges every input token at the cached rate; the second adds the uncached-token premium.

Fix the stationary block stream and its eligibility times from \cref{sec:eoq}. Define an ideal accounting reference with prompt length $L_t^0$ and uncached count $N_t^0$: it replaces every object at its first eligible request, receives cache hits for those replacements, and processes the same newly appended material. This reference is independent of batch size.

After the actual boundary decision, let $U_t$ count removable tokens still present in eligible, uncommitted objects, and let $V_t$ count additional uncached tokens caused by a commit. Then $L_t-L_t^0=U_t$ and, under stable prefix reuse, $N_t-N_t^0=V_t$. Subtracting the reference cost in \cref{eq:cost-decomposition} gives
\begin{equation}
(N_t+w_rR_t)-(N_t^0+w_rR_t^0)
=w_rU_t+(1-w_r)V_t.
\label{eq:request-excess}
\end{equation}
Thus delayed shortening and reconstruction contribute separate terms.

Before a batch of $b$ objects is committed, the wait requests carry $1,2,\ldots,b-1$ pending objects. With $r=w_rg$, their total excess is
\begin{equation}
\sum_{j=1}^{b-1}w_rjg=\frac{rb(b-1)}2=W_b.
\end{equation}
At the commit request, all pending objects become Cards, so $U_t=0$. The cached interval $I_b=(b+h-1)x$ shortens by $bg$, leaving $V_t=I_b-bg$ tokens to reconstruct. Including auxiliary overhead $F_0+bf$, the cycle excess is
\begin{align}
E_b
&=W_b+(1-w_r)(I_b-bg)+F_0+bf \nonumber\\
&=b\underbrace{\bigl[f+(1-w_r)(x-g)\bigr]}_{K}
+\underbrace{F_0+(1-w_r)(h-1)x}_{Q_h}+\frac{rb(b-1)}2 \nonumber\\
&=bK+Q_h+W_b.
\label{eq:cycle-excess}
\end{align}
Each cycle contains $b$ requests and processes $b$ eligible objects, giving excess $E_b/b=\ell_h(b)$ per request and per object. Initial and final incomplete cycles vanish from this average as the stream grows with fixed $b$.

The first-request comparison and this common-reference accounting agree because
\begin{equation}
(1-w_r)I-G=(1-w_r)(I-G)-w_rG:
\end{equation}
the reference already receives the $w_rG$ shortening benefit on the commit request, leaving $(1-w_r)(I-G)$ as excess reconstruction.

\subsection{Stationary Optimum and First Crossing}
\label{app:stationary-optimum}
\label{app:crossing-proof}

\begin{proof}[Proof of \cref{prop:eoq}]
Since $K$ is independent of $b$, minimize $Q_h/b+r(b-1)/2$. For $Q_h>0$ and $r>0$,
\begin{equation}
\ell_h'(b)=-\frac{Q_h}{b^2}+\frac r2,
\qquad
\ell_h''(b)=\frac{2Q_h}{b^3}>0.
\end{equation}
The unconstrained minimizer is $\sqrt{2Q_h/r}$; imposing $b\geq1$ gives \cref{eq:sqrt}. If $Q_h=0<r$, the objective increases with $b$, giving $b=1$.

For integer $b\geq1$, define
\begin{equation}
d_b=\ell_h(b+1)-\ell_h(b)=\frac r2-\frac{Q_h}{b(b+1)}.
\end{equation}
These differences are nondecreasing. For $b\geq2$, optimality is therefore equivalent to $d_{b-1}\leq0$ and $d_b\geq0$, or
\begin{equation}
\frac{rb(b-1)}2\leq Q_h\leq\frac{rb(b+1)}2.
\end{equation}
For $b=1$, only $d_1\geq0$ is required, giving $Q_h\leq r$; the lower bound follows from $Q_h\geq0$.
\end{proof}

Choosing the smaller batch in a tie gives
\begin{equation}
b^\dagger=
\max\left\{1,\left\lceil\frac{\sqrt{1+8Q_h/r}-1}{2}\right\rceil\right\}.
\label{eq:integer-closed-form}
\end{equation}
This is the smallest positive integer with $rb(b+1)/2\geq Q_h$. Equality gives adjacent optima $b$ and $b+1$; otherwise the optimum is unique. Under a fixed integer cap $B_{\max}$, the smallest optimal feasible size is $\min\{b^\dagger,B_{\max}\}$.

\paragraph{First crossing.}
Starting from an empty pending sequence, one object becomes eligible per request. When the next batch first contains $b$ objects, their waiting ages are $b-1,b-2,\ldots,0$, so
\begin{equation}
G_b=bg,\qquad W_b=\frac{rb(b-1)}2,\qquad
W_b+w_rG_b=\frac{rb(b+1)}2=W_{b+1}.
\end{equation}
For $r>0$, the first crossing $W_b+w_rG_b\geq Q_h$ consequently selects $b^\dagger$. It also directly establishes both optimality inequalities: the current crossing gives the upper bound on $Q_h$, while for $b>1$ the preceding noncrossing gives
\begin{equation}
W_b=W_{b-1}+w_rG_{b-1}<Q_h.
\end{equation}
For $b=1$, the lower bound is $W_1=0\leq Q_h$. Equality at the current crossing selects the smaller adjacent optimum. A confirmed full commit resets the ledger, and the argument repeats in the next cycle. This recovery uses equal object gains, one eligible block per request, constant $Q_h$ and $w_r$, and stable prefix caching. A fixed batch cap truncates the selected size to $B_{\max}$.

\paragraph{Zero-cost cases.}
If $r=0<Q_h$, $K+Q_h/b$ strictly decreases: the largest feasible batch is optimal under a fixed cap, and there is no finite optimum without a cap. If $Q_h=0<r$, immediate commitment is optimal; if both vanish, all batch sizes have equal modeled excess. With $h=1$, the cache timing gives $Q_h=F_0$, since no cached hot-tail suffix remains to reconstruct.

\subsection{Online Quantities and Ledger Updates}
\label{app:online-scope}

\paragraph{Pending region and shared suffix.}
For a candidate commit inside a reusable prefix, split the affected interval at the end of the newest pending block. Let $A_t$ cover the interval from the first edit to that boundary, including unchanged content between replacements. The remaining cached suffix has length $H_t^{\mathrm{shared}}$. Using $Q_t=F_{0,t}+(1-w_r)H_t^{\mathrm{shared}}$, the reconstruction premium is
\begin{align}
F(\mathcal P_t)+(1-w_r)(A_t-G_t+H_t^{\mathrm{shared}})
&=\underbrace{F(\mathcal P_t)-F_{0,t}+(1-w_r)(A_t-G_t)}_{\text{pending-region cost}}
+Q_t.
\label{eq:online-premium-split}
\end{align}
In the stationary model, these terms are $bK$ and $Q_h$: the pending region has constant cost per object, while batching amortizes $Q_h$. This gives the shared-cost threshold in \cref{eq:trigger}. With heterogeneous blocks, internal gaps, or nonlinear overhead, the pending-region cost can vary per object; the runtime applies the shared-cost crossing to its current estimates.

The split uses the end of the newest pending \emph{block}, so trailing dialogue within that block belongs to $A_t$. Actual edit positions also determine the interval: if the first edit starts $s$ tokens into its block, the homogeneous interval is $I_b=(b+h-1)x-s$. The stationary derivation uses $s=0$; runtime estimates use the actual edit positions and reusable cache frontier.

\paragraph{Waiting ledger.}
A request carrying pending contents raw adds $w_rG_t$ to their ledger, and newly eligible contents enter with zero accumulated loss. For fixed additive gains this gives
\begin{equation}
W_{t+1}=W_t+w_rG_t,
\qquad
W_t=w_r\sum_i g_i a_i(t),
\end{equation}
where $a_i(t)$ counts eligible requests carrying occurrence $i$ raw. A successful model response confirms a full commit and clears its entries. A retrieved tool-result occurrence starts its own exposure and waiting ledger, independently of the historical Card, and can later become a retrieval receipt. The runtime aggregates persistent per-delta entries and estimates $G_t$ from whole-prompt lengths; its recorded ledger uses these estimates rather than canonical per-object token sums.

\paragraph{Cache state.}
Reconstruction includes only tokens that would otherwise be reused, excluding natural cache misses and the newest uncached block. The increment $w_rG_t$ is the cached retention cost; for pending payloads that miss the cache, it serves as a cached-rate estimate. Irregular arrivals, retrieval, and changing $Q_t$ are handled through the current quantities. The following online result treats exact waiting costs and schedule-independent shared setup costs.

\subsection{A Competitive Bound for the Waiting--Setup Objective}
\label{app:reduced-batching-bound}

To isolate the online batching trade-off, fix requests $t=1,\ldots,T$ and each object's eligibility time $a_i\in\{1,\ldots,T\}$. An eligible object $i$ incurs a fixed cost $r_i\geq0$ on every request that retains it raw. A commit replaces all pending objects and incurs shared setup cost $Q_t$. Eligibility, object gains, and $Q_t$ are independent of the schedule; objects remain pending until committed, and either action is feasible at every boundary. Let $\mathcal U_t^\pi$ be the objects left pending after schedule $\pi$'s decision at $t$. The waiting--setup objective is
\begin{equation}
J(\pi)=\sum_{t=1}^{T}\sum_{i\in\mathcal U_t^\pi}r_i+\sum_{t:\,\pi\text{ commits at }t}Q_t.
\label{eq:reduced-batching-objective}
\end{equation}
Objects may remain raw when the history ends. This is the setup-versus-delay objective of online acknowledgment batching~\citep{dooly2001tcp}. Here it isolates waiting and shared reconstruction from the full input-cost decomposition. With additive token gains and cached retention, $r_i=w_rg_i$.

\begin{proposition}[Online waiting--setup bound]
\label{prop:reduced-batching-bound}
Suppose $0<Q_{\min}\leq Q_t\leq Q_{\max}$. Let $W_t$ be the waiting cost already incurred by the current pending batch and $\rho_t=\sum_{i\in\mathcal P_t}r_i$ its cost if retained on request $t$. The rule that commits exactly when $\mathcal P_t\ne\varnothing$ and $W_t+\rho_t\geq Q_t$ satisfies
\begin{equation}
J(\mathrm{ALG})\leq\frac{2Q_{\max}}{Q_{\min}}J(\mathrm{OPT}),
\label{eq:reduced-batching-bound}
\end{equation}
where $\mathrm{OPT}$ minimizes \cref{eq:reduced-batching-objective} with full knowledge of the history. In particular, the rule is $2$-competitive for constant positive $Q_t$.
\end{proposition}

\begin{proof}
Let $\tau_1<\cdots<\tau_m$ be the algorithm's commit times and set $\tau_0=0$. A completed phase is the request interval $(\tau_{j-1},\tau_j]$; its objects are those becoming eligible in that interval. Since every commit clears the pending set, the algorithm's phase cost is $W_{\tau_j}+Q_{\tau_j}$. If the phase has more than one request, the preceding noncrossing gives
\begin{equation}
W_{\tau_j}=W_{\tau_j-1}+\rho_{\tau_j-1}<Q_{\tau_j-1}\leq Q_{\max}.
\end{equation}
For a one-request phase, $W_{\tau_j}=0$. Thus every completed phase costs at most $2Q_{\max}$.

If OPT commits during that phase, assign one such setup cost to the phase; it is at least $Q_{\min}$. Otherwise, OPT retains all the phase's objects on every request from their eligibility through $\tau_j$. Assign their waiting cost in this interval, which is
\begin{equation}
W_{\tau_j}+\rho_{\tau_j}\geq Q_{\tau_j}\geq Q_{\min}.
\end{equation}
The term $\rho_{\tau_j}$ is the cost of the current request, so this argument also applies when $\tau_j=T$. Each completed phase therefore costs ALG at most $2Q_{\max}/Q_{\min}$ times its assigned OPT cost.

In a final uncompleted phase, ALG's total cost is zero if no objects arrive; otherwise the last noncrossing gives $W_T+\rho_T<Q_T\leq Q_{\max}$. If OPT commits in this phase, assign one setup cost, at least $Q_{\min}$. If it does not, assign the waiting cost of the phase's objects, equal to ALG's cost. This phase satisfies the same bound. All assigned costs are distinct: setups belong to disjoint phases, and waiting charges belong to distinct phase objects and request intervals. Their sum is at most $J(\mathrm{OPT})$, proving \cref{eq:reduced-batching-bound}.
\end{proof}

If pending-region reconstruction contributes the same nonnegative total $B$ to every feasible schedule, the bound also holds for $B+J$: for $\alpha=2Q_{\max}/Q_{\min}\geq1$, $B+J(\mathrm{ALG})\leq\alpha[B+J(\mathrm{OPT})]$.

\FloatBarrier
\section{Experimental Configuration and Evaluation Protocols}
\label{app:protocol}

\paragraph{Benchmark chains.}
The continuing task lists contain 89 Terminal-Bench 2.1 tasks, the first 18 evaluated LongMemEval-V2 tasks, and 50 SWE-bench Verified tasks. All runs request the \texttt{gpt-5.6-terra} route. Each benchmark uses one continuing chain per arm, with history retained across tasks. Task identities and positions align where task-level exports are available; model and tool trajectories develop separately. The evaluated scope includes processed requests on the completed task lists. Terminal-Bench includes a resumed execution, with cache continuity reconstructed at the continuation boundary. All benchmarks use \cref{eq:inputcost}; prompt accounting uses \texttt{o200k\_base} and the versioned serializer \texttt{logical-prompt-json-v1}.

\paragraph{HCE replay and aggregation.}
The three frozen HCE histories cover configuration, debugging, and feature implementation. They contain 25, 21, and 37 source-request boundaries, respectively, and 44 frozen Card occurrences in total. Object identities, immutable versions, original bytes, and prepared Cards are fixed across configurations at $h=2,4,8,16$. Prompts use \nolinkurl{canonical_json_sorted_utf8:v1}, with \nolinkurl{tiktoken:o200k_base:0.11.0} tokenization and adjacent-request prefix reuse under \nolinkurl{canonical_prefix_reference}. Reported costs sum $N+0.1R$ across all three histories. Aggregate saving is $1-\sum_j C_{p,j}/\sum_j C_{\mathrm{Full},j}$ for policy $p$ and histories $j$.

\paragraph{Hot tail and production configuration.}
All strategies use both the block cap $h$ and the 12,800-token hot-tail cap. Production CADOC additionally confirms an occurrence's initial raw exposure after a successful model response. Its token estimate covers the suffix starting at the earliest protected occurrence, whereas baseline replay logs the corresponding recent-block token sum. The exported \texttt{hot\_tokens} and \texttt{hot\_block\_count} therefore describe each replay path's partition. Effective configurations and governing source files accompany the code.

CADOC uses rough prompt/cache estimates, a persistent per-delta waiting ledger, and zero auxiliary commit overhead: $Q_t=(1-w_r)H_t^{\mathrm{shared}}$. Pending guards fire above $2h$ runtime inference groups or 25,600 raw tokens; baseline capacity guards use 16 pending blocks or 65,536 raw tokens. All 35 observed CADOC commits follow the economic crossing and are confirmed after a successful response. Fixed $b=16$, Token 32,768, and Token 65,536 make no commits on these finite HCE histories and therefore equal Full Context; the main figure shows the eight active configurations.

\paragraph{Continuation evaluation.}
Six policies are evaluated on 31 end-of-history probes, each repeated with seeds 20260907, 20260908, and 20260909. Pass counts pool the 93 evaluations. Session-macro scores average repeats within probes, probes within each history, and the three histories equally. Answering and judging request \texttt{gpt-5.6-terra}. \Cref{app:quality-grid} reports the aggregate policy comparison.

\subsection{Fixed-History Replay of Benchmark Traces}
\label{app:replay-protocol}

The benchmark replay uses $h=8$ and $w_r=0.1$ on 1,801 recorded FULL request boundaries: 1,013 spanning Terminal-Bench tasks, 691 from SWE-bench, and 97 spanning LongMemEval tasks. Source outputs and deterministic Cards are fixed, all configurations use a shared inference candidate pool, and CADOC runs its production scheduler and lifecycle. Context and cache state continue across tasks; Terminal-Bench continuation follows the retained source-task order. Each stream ends at its last observed input. The replay performs no model calls. 

\FloatBarrier
\section{Object Implementation and Recorded Retrieval}
\label{app:cards}

The Card shown in \cref{sec:substrate} comes from LongMemEval-V2 task \texttt{0f970f01}, which asks which column lies between Price and Delivery time in a ServiceNow Catalog Items list. A terminal search returns trajectory locations as a structured tool result. CADOC preserves its original bytes and replaces its historical occurrence with a deterministic Card containing a versioned object reference and structural metadata.

\paragraph{Retrieval and subsequent use.}
Request 11 contains the original result; request 12 contains its Card after an economic-crossing commit. The model then issues the recorded call:
\begin{quote}\begin{minipage}{\linewidth}\small\ttfamily\raggedright
\{"object\_ref":\\
\quad "object://obj\_0722dfa749e04b58a0094af4@v1",\\
\quad "reason": "Need the exact trajectory ID and state index returned by the evidence search for the Catalog Items column-header question."\}
\end{minipage}\end{quote}
Request 13 contains the call and its full-object result. The original payload and the returned \nolinkurl{retrieved_object.content} both contain 9,925 UTF-8 bytes and have identical recomputed SHA-256 hashes. The retrieved content identifies trajectory \texttt{3c588c61}, state 33; the next action in request 14 uses that location:
\begin{quote}\small\ttfamily python3 enterprise/show\_trajectory.py 3c588c61 --start 33 --end 39\end{quote}
The task completes correctly. Runtime estimates are 2,473 tokens for the object and 174 for its Card.

Retrieval returns the immutable referenced version as a new tool-result occurrence. The historical Card remains in place, while the returned occurrence receives its initial raw exposure and then follows the same hot-tail and commit lifecycle, including replacement by a retrieval receipt. The evidence package contains the complete Card, original and returned payloads, request snapshots, grading record, and registration/retrieval source code.

\section{Aggregate HCE Continuation Quality}
\label{app:quality-grid}

\Cref{tab:quality-grid} reports the six-policy comparison at each hot-tail setting, aggregated over the same 93 probe evaluations using the protocol in \cref{app:protocol}. Full Context uses the same unchanged control across all four settings.
\begin{table}[ht]
\centering\small
\caption{Continuation passes and session-macro scores for six policies across four hot-tail settings.}
\label{tab:quality-grid}
\setlength{\tabcolsep}{4.5pt}
\begin{tabular}{lrrrrrrrr}
\toprule
& \multicolumn{2}{c}{$h=2$} & \multicolumn{2}{c}{$h=4$} & \multicolumn{2}{c}{$h=8$} & \multicolumn{2}{c}{$h=16$}\\
\cmidrule(lr){2-3}\cmidrule(lr){4-5}\cmidrule(lr){6-7}\cmidrule(lr){8-9}
Policy & Passes & Macro & Passes & Macro & Passes & Macro & Passes & Macro\\
\midrule
Full & 57/93 & 4.318 & 57/93 & 4.318 & 57/93 & 4.318 & 57/93 & 4.318\\
Immediate & 64/93 & 4.326 & 59/93 & 4.274 & 64/93 & 4.285 & 63/93 & 4.404\\
Fixed $b=4$ & 65/93 & 4.318 & 64/93 & 4.346 & 62/93 & 4.305 & 61/93 & 4.354\\
Fixed $b=8$ & 63/93 & 4.336 & 59/93 & 4.263 & 64/93 & 4.353 & 69/93 & 4.538\\
Token 16,384 & 65/93 & 4.421 & 60/93 & 4.247 & 67/93 & 4.492 & 67/93 & 4.447\\
\textbf{CADOC} & 61/93 & 4.297 & 61/93 & 4.204 & 65/93 & 4.418 & 62/93 & 4.330\\
\bottomrule
\end{tabular}
\end{table}

\FloatBarrier

\section{Recorded Scheduling Decisions}
\label{app:runtime}
\label{app:case}
\label{app:perturbation}

\Cref{tab:crossing-case} reconstructs consecutive production decisions at $h=4$ from the runtime's $W$, $G$, and $Q$, with $G$ given by its rough keep/commit prompt difference. The inequality changes from $4,787.1<5,688.9$ to $6,274.9\geq6,124.5$, triggering a commit. That batch replaces five objects containing 15,006 raw tokens with 764 Card tokens, a 94.91\% reduction.
\begin{table}[!ht]
\centering\small
\caption{Two decisions to show the process; boundaries are zero-based request indices.}
\label{tab:crossing-case}
\setlength{\tabcolsep}{6pt}
\begin{tabular}{rrrrrl}
\toprule
Boundary & $W$ & $G$ & $Q$ & $W+0.1G$ & Decision\\
\midrule
14 & 3,424.9 & 13,622 & 5,688.9 & 4,787.1 & Wait\\
15 & 4,787.1 & 14,878 & 6,124.5 & 6,274.9 & Commit\\
\bottomrule
\end{tabular}
\end{table}

Across all 184 nonempty-pending boundaries, recomputing $W+0.1G-Q$ matches every recorded decision. Holding preceding history fixed, 172 of these decisions remain unchanged throughout independent $\pm5\%$ perturbations of the recorded $W,G,Q$. 
\FloatBarrier

\section{Sensitivity to Cache-Read Weight}
\label{app:price}

\paragraph{Protocol.}
We use the benchmark histories from \cref{app:replay-protocol} and test $w_r\in\{0.05,0.1,0.25,0.5,1\}$ with input cost $C=N+w_rR$. CADOC executes its production planner separately at each weight. The nine baselines have weight-independent triggers, so their fixed token trajectories are repriced at each weight. Within each benchmark, request boundaries, inference candidates, Card bytes, and context budgets are fixed. Each history starts with an empty cache and ends at its last recorded input.

\paragraph{Results.}
CADOC attains the lowest input cost among the ten compared policies in all fifteen benchmark--weight combinations (\cref{fig:wr-sensitivity,tab:wr-sensitivity}). It ties with Immediate at $w_r=1$ on all three histories and at $w_r=0.5$ on LongMemEval, and is uniquely best in the remaining combinations.

\paragraph{Immediate commitment without a cache discount.}
At $w_r=1$, cached and uncached input tokens have equal cost. The evaluated scheduler uses zero auxiliary commit overhead, $F_{0,t}=0$, so \cref{eq:Qonline} gives
\begin{equation}
Q_t=F_{0,t}+(1-w_r)H_t^{\mathrm{shared}}=0.
\label{eq:wr-one-setup}
\end{equation}
For a nonempty eligible batch, $W_t\geq0$ and $G_t\geq0$, so the commit condition in \cref{eq:trigger} reduces to $W_t+G_t\geq0$ and is satisfied at every feasible boundary. CADOC therefore commits as soon as replacement is permitted, matching Immediate's commit timing and input cost. Immediate is the lowest-cost competing compressor at this endpoint on all three histories; CADOC's relative advantage is consequently zero, while its savings relative to Full remain positive.

\begin{figure}[!ht]
\centering
\includegraphics[width=\linewidth]{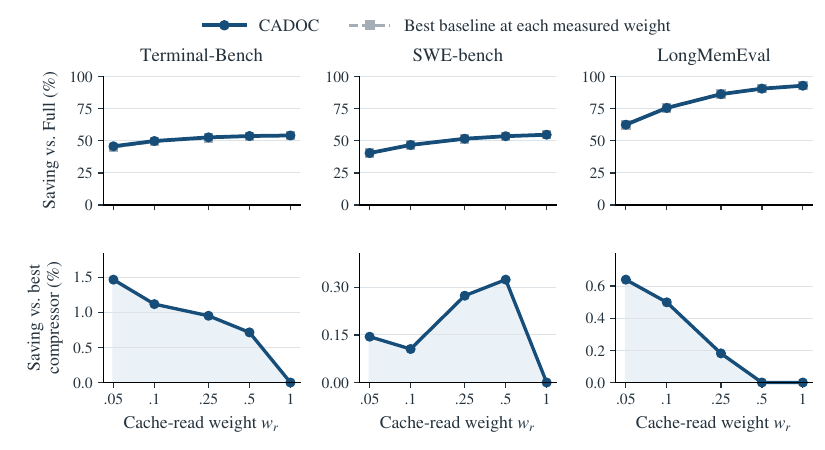}
\caption{Cache-read weight sensitivity. Top: savings relative to Full for CADOC and the best of nine baselines, including Full, at each weight. Bottom: savings relative to the best of eight other compression policies, with separate vertical scales starting at zero. Markers are measured points; lines guide the eye.}
\label{fig:wr-sensitivity}
\end{figure}

\begin{table}[ht]
\centering\small
\caption{CADOC input-cost savings relative to Full at five cache-read weights. Pooled savings use the ratio of summed costs across the three benchmark histories.}
\label{tab:wr-sensitivity}
\setlength{\tabcolsep}{10pt}
\begin{tabular}{rrrrr}
\toprule
$w_r$ & Terminal-Bench & SWE-bench & LongMemEval & Pooled\\
\midrule
0.05 & 45.73\% & 40.48\% & 62.52\% & 45.06\%\\
0.1 & 49.83\% & 46.75\% & 75.57\% & 50.39\%\\
0.25 & 52.66\% & 51.61\% & 86.28\% & 54.26\%\\
0.5 & 53.71\% & 53.57\% & 90.55\% & 55.75\%\\
1 & 54.12\% & 54.72\% & 92.85\% & 56.48\%\\
\bottomrule
\end{tabular}
\end{table}

\FloatBarrier

\end{document}